\documentclass[10 pt]{article}

\usepackage{amsmath,amsthm,amssymb,amsfonts}
\usepackage{mathtools}
\usepackage{fullpage}
\usepackage{xcolor}
\usepackage{enumitem}
\usepackage[numbers]{natbib}
\usepackage{todonotes}
\usepackage{bbm}
\usepackage[breaklinks=true,linktocpage=true]{hyperref}
\hypersetup{
    colorlinks=true,
    linkcolor={purple},
    urlcolor={purple},
    citecolor={blue!80!black}
}
\usepackage[nameinlink,noabbrev,capitalize]{cleveref}

\theoremstyle{plain}
\newtheorem{thm}{Theorem}[section]
\newtheorem{theorem}[thm]{Theorem}
\newtheorem{lemma}[thm]{Lemma}

\theoremstyle{definition}
\newtheorem{definition}[thm]{Definition}

\theoremstyle{remark}
\newtheorem{remark}[thm]{Remark}

\newcommand{\set}[1]{\{#1\}}
\newcommand{\Set}[1]{\left\{#1\right\}}
\newcommand{\eps}{\epsilon}
\newcommand{\N}{\mathbb{N}}

\newcommand{\cA}{\mathcal{A}}

\newcommand{\cG}{\mathcal{G}}
\newcommand{\cH}{\mathcal{H}}

\newcommand{\cX}{\mathcal{X}}

\DeclareMathOperator{\ldim}{\textsc{ldim}}
\DeclareMathOperator{\VCdim}{\textsc{vc}}
\DeclareMathOperator{\LR}{\textnormal{\textsc{lr}}}
\DeclareMathOperator{\loss}{loss}
\DeclareMathOperator{\supp}{supp}
\DeclareMathOperator{\SOA}{SOA}
\DeclareMathOperator{\Maj}{Maj}
\newcommand{\TV}{\mathtt{TV}}

\let\emptyset\varnothing

\title{Near-Optimal List Replicability from Littlestone Dimension}

\author{
    Ari Blondal\thanks{McGill University, \texttt{ari.blondal@mail.mcgill.ca}} \and
    Pooya Hatami\thanks{Ohio State University, \texttt{\{hatami.2, lalov.1, tretiak.2\}@osu.edu}} \and  Chavdar Lalov\footnotemark[2] \and Sivan Tretiak\footnotemark[2]
}

\begin{document}

\maketitle

\begin{abstract}
We study the list replicability number in the realizable PAC learning framework and prove that every concept class $\cH$ over a countable domain with Littlestone dimension $d$ and dual VC dimension $v^*$ satisfies
\[
\LR(\cH)\leq
2^{O(d\min\{\log d,v^*\})}.
\]
This matches the worst-case lower bound of $2^{\Omega(d)}$ up to a logarithmic factor in the exponent and exponentially improves the previous best upper bound $\LR(\cH)\le 2^{2^{O(d)}}$ of Bun, Livni, and Moran~[FOCS 2020]. Unlike earlier bounds of $2^{O_\eps(d^2)}$ for fixed accuracy $\eps>0$ implied by the work of Ghazi, Golowich, Kumar, and Manurangsi~[STOC 2021] and Lyu~[STOC 2026], our bound
is independent of $\eps$.

The $2^{O(d\log d)}$ bound follows from a new inductive framework for bounding list replicability that builds on the recent notion of irreducibility from differentially private learning. The $2^{O(dv^*)}$ bound follows from a simple globally stable learner that takes the majority vote of multiple empirical risk minimizers, with stability established using tools from stable regularity.
\end{abstract}

\section{Introduction}
In machine learning, an algorithm is \emph{replicable} if its output is robust under resampling training data. This notion has gained significant attention in recent years as part of a broad movement towards increasing trust in machine learning. To systematize the mathematical study of replicability, various formalizations of the notion have been introduced \cite{impagliazzo2023reproducibility, dixon2023listandcertificate, CMY23, hopkins2026approximate, karbasi2023reinforcement}. Subsequent work has revealed connections between these definitions themselves, and also connections to notions of stability and privacy \cite{moran2023zoo}.

Our focus is on replicability in the PAC learning framework, and specifically on the notion of \emph{list replicability}, which was introduced in \cite{dixon2023listandcertificate, CMY23}. An algorithm learns a concept class \emph{list replicably} if (with high probability) its output falls in a small list. Crucially, the list size is independent of the underlying distribution, though the hypotheses in the list can be distribution dependent. Since its introduction, list replicability has been well studied from several perspectives~\cite{CMY23, chase_local_2024, chornomaz2025spherical, blondal2026borsuk, blondal_sign-rank_2026, blondal2026simplicial, blondal_tight_2026}, with particular attention to its relationship with complexity measures of concept classes such as the VC and Littlestone dimensions.

Let $\cX$ be a countable domain. The smallest list size for which a binary concept class $\cH \subseteq \Set{\pm1}^\cX$ can be list replicably learned is called the \emph{list replicability number} of $\cH$, denoted $\LR(\cH)$ (see \cref{def:list-replicability}). A fundamental property of list replicability is that it characterizes online learnability: a concept class is online learnable if and only if it is list replicably learnable, that is, if and only if it has finite Littlestone dimension~\cite{BLM20, alon2022private, CMY23}.\footnote{These results are proved in the language of global stability, an equivalent notion whose parameter is the reciprocal of the list replicability number~\cite{CMY23}.}

The quantitative relationship between the two quantities, however, remained far from understood. Writing $d = \ldim(\cH)$ for the Littlestone dimension, the globally stable learner of Bun, Livni, and Moran~\cite{BLM20} gives the doubly exponential bound $\LR(\cH) \le 2^{2^{O(d)}}$, and it was not known whether this is optimal. In the other direction, the dual of the full concept class $\{\pm 1\}^{\cX}$ has list replicability number $\Theta(|\cX|)$~\cite{chase_local_2024} and Littlestone dimension $\lfloor \log_2|\cX|\rfloor$, so any upper bound must be at least exponential in $d$. We give an exponential improvement of the upper bound, nearly matching this lower bound.

\begin{theorem}[Informal version of \cref{thm:main1} and \cref{thm:main2}]\label{thm:main_combined}
For any countable domain $\cX$ and arbitrary class $\cH \subseteq \Set{\pm1}^\cX$ with Littlestone dimension $\ldim(\cH)=d$ and dual VC dimension $\VCdim^*(\cH)=v^*$, the following inequality holds:
\[
    \LR(\cH)\le 2^{O(d \min\{\log d,v^*\})}.
\]
\end{theorem}
The proof of our $2^{O(d\log d)}$ bound (\cref{thm:main1}) develops a new inductive approach to list replicability. One ingredient is irreducibility, a combinatorial property of Littlestone classes that captures their stability under the standard optimal algorithm (SOA). Irreducibility was introduced in \cite{GGKM21} and has found important applications in differentially private learning \cite{GGKM21,Lyu25}. Our key combinatorial observation is that the SOA hypotheses arising from irreducible subclasses of $\cH$ with positive Littlestone dimension form a class with Littlestone dimension at most $\ldim(\cH)-1$.

The second main ingredient is a new transfer lemma, which allows us to bound the list replicability number of $\cH$ in terms of that of an auxiliary concept class, with an additive overhead in list size. This lemma provides a general framework for reducing list-replicable learning to simpler learning problems. Combining these two ingredients, we construct a list-replicable learner by induction on Littlestone dimension.

The proof of the $2^{O(dv^*)}$ bound (\cref{thm:main2}) goes through the notion of global stability, which is known to be equivalent to list replicability \cite{CMY23}. We run a proper $\operatorname{ERM}$ algorithm $O(v^*)$ many times and then take the majority of the outputs. To prove this algorithm is globally stable, we use simple tools from stable regularity, which has recently found close connections with learning theory \cite{MM24,malliaris2025epsilon,malliaris2014regularity}.

We believe the techniques in both proofs could be useful for the broader study of list replicability and related notions in stability, replicability, and privacy.

\subsection{Applications and related work}

\paragraph{Number of random bits.}
Our bounds also have implications for the number of random bits needed to achieve $\rho$-replicability in the sense of \cite{impagliazzo2023reproducibility}. Hopkins and Moran~\cite{hopkins25a} showed that a globally stable learner with stability parameter $\eta$ can be converted into a $\rho$-replicable algorithm by using $O(\log(1/\eta)+\log(1/\rho))$ shared random bits. Using the equivalence between list replicability and global stability, \cref{thm:main_combined} implies that we can $\rho$-replicably learn a concept class by using
\[O(d \min\{\log d,v^*\}+\log (1/\rho))\]
random bits independent of accuracy and confidence parameters. Again, this nearly matches the linear worst-case lower bound that follows from the list replicability lower bound above.

Similarly, \cref{thm:main_combined} implies upper bounds on the number of random bits necessary for differentially private learning \cite[Theorem 2.2 and Theorem 2.3]{hopkins25a}. 

\paragraph{Fixed-accuracy regime.}
Earlier work of Ghazi, Golowich, Kumar, and Manurangsi~\cite{GGKM21} and Lyu~\cite{Lyu25} implies the bound $\LR_\eps(\cH)\le 2^{O_\eps(d^2)}$ for every fixed accuracy $\eps>0$ (see \cref{def:list-replicability}). Unlike our result, these bounds are not uniform in $\eps$.

\paragraph{Concurrent work.}
In concurrent and independent work, Moran and Yan~\cite{moran2026nearly} used related combinatorial techniques to obtain a nearly linear dependence on Littlestone dimension in the sample complexity of private learning.

\subsection{Paper organization}

In \cref{sec:preliminaries} we introduce and discuss the definitions and theorems needed for our main results, presented in \cref{sec:main_results}. We give an overview of our proofs in \cref{sec:overview}. The full proofs are in \cref{sec:thm_main1} and \cref{sec:thm_main2}.

\subsection{AI methodology}

The authors did not use AI to formulate the questions guiding this paper, nor did we use AI in proof development.

After drafting the manuscript, we used AI to suggest improvements to the writing and to assist with checking technical details. 
All choices of presentation, organization, and exposition are our own, and all AI-assisted revisions were reviewed and edited by the authors as needed.

\section{Preliminaries}\label{sec:preliminaries}
In this section, we collect the core notation and definitions.

\paragraph{Concept classes.}
Throughout this paper $\cX$ is a nonempty countable domain and $\cH\subseteq\{\pm1\}^{\cX}$ is a concept class. Each $h\in \cH$ is called a concept and represents a labeling of $\cX$ by $\pm1$. We impose no finiteness or countability assumption on $\cH$ itself.
For a point $x\in\cX$, and $b\in\set{\pm1}$, let
\[
                 \cH|_{(x,b)}\coloneqq\set{h\in\cH:h(x)=b}
\]
be the restriction containing all concepts labeling $x$ by $b$.

\paragraph{Learning setting.}
For a distribution $D$ on $\cX\times\{\pm1\}$ and a hypothesis $h\in\set{\pm1}^{\cX}$, define the \emph{population loss} of $h$ to be
\[
\loss_D(h)=\Pr_{(x,y)\sim D}[h(x)\ne y].
\]
We call $D$ \emph{realizable} by $\cH$ if $\inf_{h\in\cH}\loss_D(h)=0$.

Let $D_h$ denote the distribution with the same $\cX$-marginal as $D$ in which every point $x$ is labeled by $h(x)$. Note that
\begin{equation}\label{eq:correction}
  \loss_D(h) = \TV(D,D_h),
\end{equation}
where $\TV$ denotes total variation distance.

A \emph{learning rule} is a (possibly randomized) function $\cA$ that maps any \emph{sample}
$S \in \bigcup_{n=0}^\infty (\cX \times \Set{\pm 1})^n$ 
to a \emph{hypothesis} $\cA(S) \in \Set{\pm 1}^{\cX}$.

In PAC learning, for every accuracy $\eps\in(0,1)$ and confidence parameter $\delta\in(0,1)$, the learner receives $n=n(\cH,\eps,\delta)$ independent examples from a realizable distribution $D$ and must output a hypothesis $h\in\set{\pm1}^{\cX}$ with loss at most $\eps$ with probability at least $1-\delta$. Outputs may lie outside $\cH$; such learning is called \emph{improper}. The minimum function $n(\cH,\epsilon,\delta)$ necessary to achieve PAC learning is called the \emph{sample complexity} of $\cH$.

\paragraph{VC dimension.}
Ordinary PAC learnability is characterized by finite \emph{VC dimension} which we now define. We say $\cH$ \emph{shatters} a set $A\subseteq\cX$ if $\set{\pm1}^{A}=\set{h|A\,:\;h\in \cH}$. The VC dimension of $\cH$ is then defined to be
\[
\sup\set{|A|:\;A\,\text{is shattered by}\,\cH}.
\]
The Fundamental Theorem of Statistical Learning states that we can use any proper Empirical Risk Minimization (ERM) rule to PAC learn $\cH$ with sample complexity 
\[n=O\left(\frac{\VCdim(\cH)\log(1/\epsilon)+\log(1/\delta)}{\epsilon}\right).\]

\paragraph{Dual class.}
The \emph{dual class} $\cH^*$ of $\cH$ is obtained by swapping the roles of hypotheses $h \in \cH$ and points $x \in \cX$, or equivalently, by transposing the matrix $(h(x))_{h \in \cH, x \in \cX}$.

We define the \emph{dual VC dimension} of $\cH$, denoted $\VCdim^*(\cH)$, to be the VC dimension of $\cH^*$. By Assouad's bound, we have
\[\VCdim^*(\cH)\leq 2^{\VCdim(\cH)+1}-1.\]

\paragraph{Global stability and list replicability.}
Bun, Livni, and Moran~\cite{BLM20} introduced global stability as a tool in proving the equivalence of private and online learning. A globally stable learner assigns some constant probability to one fixed accurate hypothesis for each distribution.

\begin{definition}[\cite{BLM20},\cite{CMY23}]\label{def:global-stability}
A learning rule $\cA$ is \emph{$(\eps,p)$-globally stable} for $\cH$ if there exists an $n$, such that for every realizable distribution $D$, there exists a hypothesis $h_D$ such that
\[
\loss_D(h_D)\le\eps
\qquad\text{and}\qquad
\Pr_{S\sim D^n}[\cA(S)=h_D]\ge p.
\]
Let $\rho_\eps(\cH)$ be the supremum of all such $p$, and define the \emph{global stability parameter} of $\cH$ to be
\[
\rho(\cH)=\inf_{\eps\in(0,1)}\rho_\eps(\cH).
\]
\end{definition}

For small $p$, global stability does not give an accurate output with high probability. Nevertheless, Chase, Moran, and Yehudayoff~\cite{CMY23} showed that global stability is equivalent to the existence of a learner in which the output belongs to a small list of accurate hypotheses.

\begin{definition}[\cite{CMY23,dixon2023listandcertificate}]\label{def:list-replicability}
Fix a concept class $\cH$, $\eps\in(0,1)$ and an integer $L\ge1$. A learning rule is an \emph{$(\eps,L)$-list replicable} learner for a class $\cH$ if, for every $\delta\in(0,1)$, there exists a sample size $n$, such that, for every realizable distribution $D$, there is a list $\Lambda_D\subseteq\{\pm1\}^{\cX}$ of size at most $L$ satisfying
\[\Pr_{S\sim D^n}[\cA(S)\in\Lambda_D]\ge1-\delta, \quad
\loss_D(h)\le\eps\quad(h\in\Lambda_D)
.
\]
Write $\LR_\eps(\cH)$ for the least such $L$, or $\infty$ if none exists, and define the \emph{list replicability number} of $\cH$ to be
\[
\LR(\cH)=\sup_{\eps\in(0,1)}\LR_\eps(\cH).
\]
\end{definition}

The following quantitative equivalence lets us pass between the two notions.
\begin{theorem}[{\cite[Theorem 7]{CMY23}}]\label{thm:stability-list}
For any concept class $\cH$ over a countable domain $\cX$ we have 
$$\rho(\cH)=1/\LR(\cH).$$
\end{theorem}

\paragraph{Littlestone dimension and SOA.}
In their celebrated paper, \cite{alon2022private} show that global stability (and therefore also list replicability) is equivalent to the finiteness of a combinatorial parameter called Littlestone dimension, which we now define. 

A \emph{mistake tree} of depth $d$ over a domain $\cX$ is a complete
binary tree whose internal nodes are labeled by points in $\cX$.
The two edges leaving each internal node are labeled $-1$ (left)
and $+1$ (right). A root-to-leaf path therefore specifies a sequence
$(x_1,b_1),\ldots,(x_d,b_d)$, where $x_i$ denotes the point at
level $i$ and $b_i \in \set{\pm 1}$ is the label of the edge followed.

A concept class $\cH \subseteq \set{\pm 1}^\cX$ \emph{shatters}
a mistake tree if, for each root-to-leaf path
$(x_1,b_1),\ldots,(x_d,b_d)$, there is an $h \in \cH$ satisfying
$h(x_i)=b_i$ for every $i \in [d]$. The \emph{Littlestone dimension} $\ldim(\cH)$ is the largest $d$ for which some depth-$d$ mistake tree is shattered, or $\infty$ if no such $d$ exists.

We have $\VCdim(\cH) \leq \ldim(\cH)$. Indeed, if
$A=\Set{x_1,\ldots,x_d}$ is shattered by $\cH$, then the depth-$d$
mistake tree whose nodes at level $i$ are all labeled by $x_i$
is also shattered by $\cH$.

We will need the prediction rule of the standard optimal algorithm (SOA), which labels each point by the restriction of larger Littlestone dimension:
\[
 \SOA_{\cH}(x)
 =
 \begin{cases}
  1,&\ldim(\cH|_{(x,1)})\geq\ldim(\cH|_{(x,-1)}),\\
 -1,&\ldim(\cH|_{(x,1)})<\ldim(\cH|_{(x,-1)}).
 \end{cases}
\]
The branch opposite to the SOA label always loses dimension:
\begin{equation}\label{eq:opposite-drops}
 \ldim\bigl(\cH|_{(x,-\SOA_{\cH}(x))}\bigr)<\ldim(\cH).
\end{equation}

\section{Main results}\label{sec:main_results}

\begin{theorem}\label{thm:main1}
Let $\cH\subseteq\set{\pm1}^{\cX}$ be a concept class over a countable domain $\cX$ and
let $d=\ldim(\cH)$. Then for any $\epsilon>0$ there is an $(\epsilon,(d+2)^d)$-list replicable learner for $\cH$ with sample complexity 
$2^{O(d^2 \log d)}\epsilon^{-2}\delta^{-(d+1)}$, for confidence $1-\delta$.
In particular, 
\[
                         \LR(\cH)\leq(d+2)^d=2^{O(d\log d)}.
\]
\end{theorem}

Note that the dual of the full concept class $\{\pm 1\}^{\cX}$ has list replicability number $\Theta(|\cX|)$ (\cite{chase_local_2024}) and Littlestone dimension $\lfloor\log_2|\cX|\rfloor$. Thus, the bound in \cref{thm:main1} is optimal up to a factor of $\log d$ in the exponent.

We also show that if the VC dimension of our concept class is constant, then that $\log$ factor can be removed.

\begin{theorem}\label{thm:main2}
Let $\cH\subseteq\set{\pm1}^{\cX}$ be a concept class over a countable domain $\cX$ and
let $d=\ldim(\cH)$, $v=\VCdim(\cH)$,  and $v^*=\VCdim^*(\cH)$  be positive integers. Then for any $\epsilon>0$ there is an $(\epsilon,\frac{1}{2^{O(dv^*)}})$-globally stable learner for $\cH$ with sample complexity
\[
O\left(
    \frac{vv^*\log(1/\epsilon)}
         {\epsilon}
\right).\]
In particular
\[
     \rho(\cH) \geq \frac{1}{2^{O(dv^*)}}
     \quad\text{and}\quad
     \LR(\cH) \leq 2^{O(dv^*)}.
\]
\end{theorem}
\begin{remark}
  Examining more closely the proof of \cref{thm:stability-list} that can be found in \cite{CMY23}, we can also compute the sample complexity of an $(\epsilon,2^{O(dv^*)})$-list replicable learner. In particular, look at the definitions of $n_0,n_1,n_2$, and $T$ in the proof of \cite[Proposition 2]{CMY23}. We conclude that
$2^{O(dv^*)}\epsilon^{-2}
    \log\frac{1}{\delta}$
samples suffice.
\end{remark}

\section{Technical overview}\label{sec:overview}

\paragraph{Proof of \cref{thm:main1}.} We introduce a new inductive method for bounding the list replicability number. A transfer lemma reduces list-replicably learning $\cH$ to learning an auxiliary class, and since SOA hypotheses of irreducible subclasses of $\cH$ form a class of strictly smaller Littlestone dimension, we can induct on the Littlestone dimension. 

The argument goes as follows. We run ERM on several independent samples to obtain an empirical distribution $\mu$
over accurate hypotheses in $\cH$. If $\mu$ has a heavy hypothesis,
of mass at least $1/(d+1)^d$, we output it.
Otherwise, we use the absence of heavy hypotheses to find a subclass
$\cG_{\mu}\subseteq\supp(\mu)$ of size $|\cG_{\mu}|\geq 2$ satisfying a combinatorial property called
$d$-irreducibility (see \cref{lem:heavy-component}).

Irreducibility was introduced by Ghazi et al.~\cite{GGKM21} and
further developed by Xin Lyu~\cite{Lyu25} to study private learning
of Littlestone classes. We formally define irreducibility in \cref{sec:SOA_irreduciblity}. For now one can think of it as a notion of stability.  We show in \cref{lem:positive-closure} that the class
$$\widetilde{\cH}\coloneqq \set{\SOA_{\cG}\;|\;\cG \subseteq \cH, \;|\cG|\geq 2,\;\cG\;\text{is}\;\text{$d-$irreducible}}$$
satisfies
\begin{equation}\label{eq:irreduciblity_littlestone}
    \ldim(\widetilde{\cH})\le d-1.
\end{equation}
We interpret \eqref{eq:irreduciblity_littlestone} as combinatorial
stability: SOA aggregation produces a class of smaller Littlestone
dimension.

By \cref{lem:positive-closure}, we can thus aggregate hypotheses from $\mu$ into a hypothesis $\SOA_{\cG_{\mu}}$ belonging to a class of smaller Littlestone dimension; moreover, $\SOA_{\cG_{\mu}}$ inherits accuracy from the hypotheses in $\cG_{\mu}$ (\cref{lem:soa-accuracy}). This allows us to study list replicability inductively through the Littlestone dimension. After
obtaining $\SOA_{\cG_{\mu}}$, we draw a fresh sample $S\coloneqq((x_1,y_1),\dots,(x_m,y_m))$ from $D$ and
relabel with $\SOA_{\cG_{\mu}}$ to get \[\widetilde{S}\coloneqq((x_1,\SOA_{\cG_{\mu}}(x_1)),\dots,(x_m,\SOA_{\cG_{\mu}}(x_m))).\]
The resulting sample can be thought of as a sample from the distribution $D_{\SOA_{\cG_{\mu}}}$, which keeps the same $\cX$-marginal of $D$ but uses labels given by $\SOA_{\cG_{\mu}}$. Note that now $D_{\SOA_{\cG_{\mu}}}$ is realizable by $\widetilde{\cH}$. Hence, we can use the list replicable learner for $\widetilde{\cH}$ on the sample $\widetilde{S}$ to output our final hypothesis. 

In a nutshell, we either output a heavy proper hypothesis or instead use the list replicable learner of a lower-dimensional Littlestone class which serves as the source of improper hypotheses for our learner. We prove that this procedure leads to a list replicable algorithm for $\cH$ with list size bounded by 
\[
(d+1)^d+(d+1)^{d-1}\le(d+2)^d,
\]
where the first term bounds the number of possible heavy proper
outputs, and the second comes from induction. 

The complete proof is presented in \cref{sec:thm_main1}.
\cref{lem:transfer} states the general framework for studying list replicability inductively; in \cref{sec:SOA_irreduciblity} we develop the combinatorics of irreducibility needed to apply it.

\paragraph{Proof of \cref{thm:main2}.}
Our second proof constructs a globally stable learner. The list
replicability bound then follows from \cref{thm:stability-list}.

The algorithm is surprisingly simple: run a deterministic proper ERM on $m=O(v^*)$ independent samples $S_1,\ldots,S_m$ and output
their pointwise majority,
\[
    \operatorname{Maj}\bigl(
        \operatorname{ERM}(S_1),\ldots,
        \operatorname{ERM}(S_m)
    \bigr).
\]

To analyze this learner, we use simple tools from stable regularity, a field which has recently found new connections with learning theory~\cite{MM24,malliaris2025epsilon,MALLIARIS2025103633}. Fix a realizable distribution $D$, and let $\mu$ be the output
distribution of one ERM run. We show, using a result of Malliaris and Moran \cite{MM24}, that there exists a subset of accurate hypotheses
$A_\mu\subseteq\supp(\mu)$ such that
\[
    \mu(A_\mu)\ge 2^{-O(d)}
\]
and the conditional distribution
$\nu=\mu(\,\cdot\mid A_\mu)$ satisfies
\[
    \Pr_{h\sim\nu}[h(x)=1]
    \in [0,1/4]\cup[3/4,1]
    \qquad\text{for every }x\in\cX.
\]
In other words, at every point, at least three quarters of the
mass under $\nu$ agrees on a label.

Uniform convergence for the dual
class shows that sampling $m=O(v^*)$ hypotheses independently
from $\nu$ and taking their majority gives with constant probability the same hypothesis.

Finally, the probability that we sample from $A_\mu$ all $m$ times is
\[
  \mu(A_\mu)^m
    \ge \bigl(2^{-O(d)}\bigr)^{O(v^*)}
    =2^{-O(dv^*)}.
\]
Thus, our learner outputs a fixed accurate hypothesis with probability \(2^{-O(dv^*)}\), establishing global stability and the desired list replicability bound.

\section{Proof of \cref{thm:main1}}\label{sec:thm_main1}

\subsection{Transfer lemma}
The following lemma describes a framework for bounding the list replicability number of a concept class $\cH$ via the list replicability number of an auxiliary concept class $\widetilde{\cH}$. The class $\widetilde{\cH}$ will be the source of the improper hypotheses used by our list replicable learner for $\cH$. When $\widetilde{\cH}$ is less complex than $\cH$ (e.g. has smaller Littlestone dimension), this lemma leads to strong bounds on the list replicability number.

\begin{lemma}[Transfer lemma]\label{lem:transfer}
Let $\cH,\widetilde{\cH} \subseteq \set{\pm 1}^\cX$ be concept classes over a countable domain $\cX$, where $\cH$ has finite VC dimension, and let $N\ge1$ be an integer and $c\geq 1$. Assume there is a fixed selection rule that assigns to every finitely supported probability distribution $\mu$ on $\cH$ a hypothesis $h_0 = h_0(\mu) \in \set{\pm1}^{\cX}$ satisfying one of the following:
\begin{enumerate}[label=(\roman*)]
  \item \label{item:heavy_proper}$h_0\in\supp(\mu)$ and $\mu(h_0)\ge 1/N$;
  \item \label{item:transfer_accuracy} $h_0\in \widetilde{\cH}$ and for every distribution $D$ on $\cX \times \set{\pm 1}$, we have 
  $$  \loss_D(h_0)\le c\max_{h\in\supp(\mu)}\loss_D(h).$$
\end{enumerate}
Then $\LR(\cH)\le N+\LR(\widetilde{\cH})$.
\end{lemma}

\begin{proof}
   Fix $\epsilon,\delta\in(0,1)$. Choose a list learner $\cA$ for
$\widetilde{\cH}$ with list size $L=\LR(\widetilde{\cH})$, accuracy $\epsilon/2$, and failure probability $\delta/4$. Let $m\ge1$ be its sample size.  

Now choose $k$ sufficiently large such that for any countably supported probability distribution $\lambda$, the empirical distribution $\widehat{\lambda}$ of $k$ independent draws from $\lambda$ satisfies
\[
  \sup_{h \in \supp(\lambda)} \bigl|\widehat{\lambda}(h)-\lambda(h)\bigr| \le \frac{1}{2N(N+1)}
\]
with probability at least $1-\delta/4$.\footnote{Such a $k$ exists, independently of the distribution, by the VC uniform convergence theorem applied to the class of singletons, which has VC dimension one.}

Set $\alpha=\min(\frac{\epsilon}{4c},\frac{\delta}{8mc})$. Finally, fix a deterministic proper ERM for $\cH$ and choose a sample size $n$ guaranteeing it error at most $\alpha$ and failure probability at most $\delta/(4k)$.

 Having chosen the necessary parameters, we present our list-replicable algorithm.

 \begin{center}
\fbox{%
\begin{minipage}{0.93\linewidth}
\textbf{List-replicable learner for $\cH$.}

\begin{enumerate}
    \item Draw $kn+m$ independent labeled examples.
    Partition the first $kn$ examples into batches
    $S_1,\dots,S_k$, each of size $n$, and reserve the remaining $m$ samples in $T=((x_1,y_1),\dots,(x_m,y_m))$.

    \item Run the ERM for $\cH$ on each batch,
    obtaining $h_1,\ldots,h_k$, and form the empirical
    distribution
    \[
        \widehat{\mu}(h)
        =\frac{1}{k}\sum_{i=1}^{k}\mathbbm{1}_{h=h_i}.
    \]

    \item Apply the fixed selection rule to
    $\widehat{\mu}$ to obtain a hypothesis $h_0$.

    \item\label{item:output_heavy} If
    $\widehat{\mu}(h_0)\ge 1/N$, output $h_0$.

    \item \label{item:auxiliary_output} Otherwise, $h_0\in\widetilde{\cH}$ and satisfies \ref{item:transfer_accuracy}. Relabel the reserved
    examples to obtain
    \[
        T_{h_0}
        =((x_1,h_0(x_1)),\ldots,(x_m,h_0(x_m))),
    \]
    and output $\cA(T_{h_0})$.
\end{enumerate}
\end{minipage}%
}
\end{center}

We now show accuracy and replicability for our learner. Fix an arbitrary realizable distribution $D$. Let $\lambda_D$ denote the probability distribution of
one ERM output from an $n$-sample from $D$. This distribution is countably supported
because its input space is countable.

With probability at least
$1-\delta/2$, the following event, which we denote by $E$, holds:
\begin{equation}\label{eq:good-event}
  \max_{i\le k}\loss_D(h_i)\le\alpha,
  \qquad
  \sup_{h\in\supp(\lambda_D)}|\widehat\mu(h)-\lambda_D(h)|\le \frac{1}{2N(N+1)}.
\end{equation}
Indeed, the first inequality holds with probability at least $1-\delta/4$
by a union bound over the $k$ ERM runs. The second holds with
probability at least $1-\delta/4$ by the choice of $k$.

We now show that with high probability the output of our algorithm belongs to one of two small lists of accurate hypotheses. 

\paragraph{Heavy proper hypotheses.}
On $E$, every output from step \ref{item:output_heavy} belongs to the fixed set
\[
  P_D=\left\{h:\lambda_D(h)>\frac1{N+1},\ \loss_D(h)\le\alpha\right\},
\]
since \eqref{eq:good-event} gives $\loss_D(h_0)\le\alpha$ and $\lambda_D(h_0)\ge\frac{1}{N}-\frac{1}{2N(N+1)}>\frac1{N+1}$. Note that $|P_D|\le N$.

\paragraph{The auxiliary list.}On $E$, whenever the algorithm reaches step
\ref{item:auxiliary_output}, the selected hypothesis
$h_0\in\widetilde{\cH}$ satisfies $\loss_D(h_0)\le c\alpha$.
If $\widetilde{\cH}$ contains no hypothesis satisfying this
bound, then this step cannot occur on $E$ (and so our algorithm is already list-replicable with list size at most $|P_D|\le N$).

Otherwise, fix any $h_0'\in\widetilde{\cH}$ with
$\loss_D(h_0')\le c\alpha$. We will use $h_0'$ to construct
a fixed list $Q_D$, and show that this same list works for
every auxiliary hypothesis $h_0$ selected when $E$ holds.

Let the distributions $D_{h_0}$ and $D_{h_0'}$ have the same $\cX$-marginal as $D$ but relabeled with $h_0$ and $h_0'$ respectively. Note $D_{h_0'}$ is realizable by $\widetilde{\cH}$, so there is a fixed list $Q_D$ of size at most $L$ such that 
\begin{equation}
    \Pr_{S\sim D^m_{h_0'}}(\cA(S)\in Q_D)\ge 1-\delta/4,\quad\loss_{D_{h_0'}}(g)\leq\epsilon/2~~(g\in Q_D).
\end{equation}

We show that running $\cA$ on samples from $D_{h_0}$ instead of $D_{h_0'}$ still outputs from $Q_D$ with high probability. Indeed, the total variation distance
\begin{align*}
\TV(D^m_{h_0'},D^m_{h_0}) & \leq m \TV(D_{h_0'},D_{h_0}) \\
&\le m(\TV(D_{h_0'},D)+\TV(D_{h_0},D)) \\ 
&= m(\loss_{D}(h_{0}')+\loss_{D}(h_{0})) \\
& \leq 2 m c\alpha \leq \delta/4.
\end{align*}
Hence, 
\[ \Pr_{S\sim D^m_{h_0}}(\cA(S)\in Q_D)\ge 1-\delta/2.\]

We now show accuracy for any hypothesis $g\in Q_D$:
\begin{align*}
    \loss_D(g)  & \le \loss_{D_{h_0'}}(g)+\TV(D,D_{h_0'}) \\
                & = \loss_{D_{h_0'}}(g)+\loss_{D}(h_0') \qquad \text{(by \eqref{eq:correction})} \\
                & \le \epsilon/2+c\alpha  \le \epsilon.
\end{align*}
 
Putting everything together: on $E$, if the algorithm halts at step \ref{item:output_heavy}, its output lies in $P_D$. Otherwise, in step \ref{item:auxiliary_output} it outputs an element of $Q_D$ except with probability $\delta/2$. Since $\Pr[E^c]\le\delta/2$, the output lies in $P_D\cup Q_D$ with probability at least $1-\delta$. We have shown that $|P_D\cup Q_D|\le N+L$ and that every hypothesis in this list has loss at most $\epsilon$. Thus the algorithm is an $(\epsilon, N+L)$-list replicable learner for $\cH$.
\end{proof}

\subsection{SOA and irreducibility}\label{sec:SOA_irreduciblity}

Here we collect key lemmas about the combinatorics of Littlestone classes that will be used in the proof of \cref{thm:main1}.

Recall the definition of the SOA rule, where $\SOA_{\cH}(x)=1$ if and only if 
$\ldim(\cH|_{(x,1)})\geq\ldim(\cH|_{(x,-1)})$, and recall that the opposite branch always drops dimension \eqref{eq:opposite-drops}.
Our learner will output SOA hypotheses of subclasses of accurate hypotheses, so we first record a bound on the loss of the SOA rule.
\begin{lemma}[Accuracy of SOA]\label{lem:soa-accuracy}
Let $\cH$ be a concept class, let $d=\ldim(\cH)$, and let $D$ be any
distribution on $\cX \times \set{\pm1}$. Then
\[
        \loss_D(\SOA_{\cH})
        \leq 2^d\sup_{h\in\cH}\loss_D(h).
\]
\end{lemma}

\begin{proof}
Choose a depth-$d$ tree shattered by $\cH$, and choose one
$h_\sigma\in\cH$ for each leaf $\sigma\in\set{\pm1}^d$. The class
\[
                    \mathcal{G}=\set{h_\sigma:\sigma\in\set{\pm1}^d}
\]
still shatters this tree, and hence $\ldim(\mathcal{G})=d$.

For every $x\in\cX$, at least one member of $\mathcal{G}$
agrees with $\SOA_{\cH}(x)$. Otherwise
$\mathcal{G}\subseteq\cH|_{x=-\SOA_{\cH}(x)}$, contradicting
\eqref{eq:opposite-drops}. Consequently, for every labeled point $(x,y)$,
\[
        \mathbbm{1}[\SOA_{\cH}(x)\neq y]
        \leq
        \sum_{\sigma\in\set{\pm1}^d}
        \mathbbm{1}[h_\sigma(x)\neq y].
\]
Taking expectations proves the result.
\end{proof}

Another ingredient is \emph{irreducibility}, a notion  introduced by Ghazi, Golowich, Kumar, and
Manurangsi~\cite[Definition~4.1]{GGKM21}. We use the formulation of Lyu~\cite[Definition~4.1]{Lyu25}.

\begin{definition}[Irreducibility]\label{def:irreducibility}
Let $k\in\N$. A concept class $\cG$ is $k$-\emph{irreducible} if, for
every $x_1,\ldots,x_k\in\cX$,
\[
 \ldim\bigl(
   \cG|_{(x_1,\SOA_{\cG}(x_1)),\ldots,
           (x_k,\SOA_{\cG}(x_k))}
 \bigr)
 =\ldim(\cG).
\]
Otherwise $\cG$ is $k$-\emph{reducible}.
\end{definition}

For $k\in\N$, define the \emph{$k$-SOA closure} of $\cH$ by
\[
 \widehat{\cH}_k
 \coloneqq
 \Set{
   \SOA_{\cG}:
   \emptyset\neq\cG\subseteq\cH,
   \ \cG\text{ is $k$-irreducible}
 }.
\]

\begin{lemma}[{\cite[Lemma 4.3]{GGKM21}}]
\label{lem:classical-closure}
If $\ldim(\cH)\leq d$, then
\[
                         \ldim(\widehat{\cH}_{d+1}) \leq d.
\]
If $\ldim(\cH)=d$, then equality holds because
$\cH\subseteq\widehat{\cH}_{d+1}$.
\end{lemma}

We need a small modification that isolates only the classes with positive Littlestone dimension. Set
\begin{equation}\label{eq:double-hat}
\widetilde{\cH}_k
 \coloneqq
 \Set{
   \SOA_{\cG}:
   \cG\subseteq\cH,\ 
   \ldim(\cG)\geq1,\ 
   \cG\text{ is $k$-irreducible}
 }.
\end{equation}

\begin{lemma}
\label{lem:positive-closure}
If $\ldim(\cH)\leq d$ and $d\geq1$, then $\ldim(\widetilde{\cH}_d)
                  \leq d-1.
$
\end{lemma}

\begin{proof}
Suppose instead that $\widetilde{\cH}_d$ shatters a complete
depth-$d$ tree. For every leaf $\sigma$, choose
\[
                 f_\sigma=\SOA_{\cG_\sigma}
                 \in\widetilde{\cH}_d
\]
which realizes the corresponding branch. The $d$ labels along this branch
are precisely the fixed SOA labels of $\cG_\sigma$. Since
$\cG_\sigma$ is $d$-irreducible, restricting by all these labels leaves its
positive Littlestone dimension unchanged. The surviving restriction
therefore shatters a depth-one tree.

Attach such a final split below each leaf $\sigma$. Every witnessing concept
belongs to $\cG_\sigma\subseteq\cH$, so this produces a complete
depth-$(d+1)$ tree shattered by $\cH$, contradicting
$\ldim(\cH)\leq d$.
\end{proof}

\begin{remark}Note 
$\widehat{\cH}_d =\cH\cup\widetilde{\cH}_d$. Thus $\widetilde{\cH}_d$ is the potentially improper part of the $d$-SOA closure and all genuinely improper hypotheses introduced by the classical closure lie in a class of Littlestone dimension at most $d-1$.
\end{remark}

The following argument is inspired by Fact~2.8 and Lemma~6.5 in~\cite{MM24}. However, we substitute the notion of goodness in~\cite{MM24} with that of irreducibility.

\begin{lemma}[Heavy irreducible component]\label{lem:heavy-component}
Let $d\geq1$, let $\cH$ be a nonempty concept class with $\ldim(\cH)\leq d$, and let $\mu$
be a finitely supported probability distribution on $\cH$. There is a nonempty
$d$-irreducible class $\cG\subseteq\supp(\mu)$
such that
\[
                         \mu(\cG)\geq\frac{1}{(d+1)^d}.
\]
\end{lemma}

\begin{proof}
Start with $\cG_0=\supp(\mu)$. If $\cG_0$ is $d$-irreducible, stop. Otherwise,
let $x_1,\ldots,x_d$ witness its $d$-reducibility and put
$f=\SOA_{\cG_0}$. The $d+1$ classes
\[
 \cG_0|_{(x_i,-f(x_i))}\quad (1\leq i\leq d),
 \qquad
 \cG_0|_{(x_1,f(x_1)),\ldots,(x_d,f(x_d))}
\]
cover $\cG_0$, and every member has Littlestone dimension strictly
smaller than $\ldim(\cG_0)$.

At least one member of the cover has $\mu$-mass at least
$\mu(\cG_0)/(d+1)$. Replace $\cG_0$ by a heaviest such member, denote it by $\cG_1$ and continue.
The Littlestone dimension drops at every step, so there are at most $d$
steps. The final class is $d$-irreducible and retains mass at least
$(d+1)^{-d}$. Notice that if the descent reaches dimension zero, the final
class is automatically a singleton and is trivially $d$-irreducible.
\end{proof}

\subsection{Finishing the proof}
Induct on the Littlestone dimension. A class with Littlestone dimension zero has
one hypothesis, so its list replicability number is $1$. For the inductive step, put
$N=(d+1)^d$ and $\widetilde{\cH}=\widetilde{\cH}_d$. By induction and \cref{lem:positive-closure}, we have $\LR(\widetilde{\cH})\le (d+1)^{d-1}$.

Given a finitely supported distribution $\mu$ on $\cH$, choose $\cG \subseteq \supp \mu$ from \cref{lem:heavy-component}. If $\ldim(\cG)=0$, then $\cG$ is a single hypothesis with mass at least $1/N$. Otherwise, $\ldim(\cG)\geq 1$ and so $\SOA_{\cG}\in \widetilde{\cH}$. By \cref{lem:soa-accuracy}, we have 
\[
        \loss_D(\SOA_{\cG})
        \leq 2^d\max_{g\in\cG}\loss_D(g)\leq 2^d\max_{g\in \supp(\mu)}\loss_D(g) \quad \text{for any }D.
\]

By \cref{lem:transfer}, we conclude
\[\LR(\cH)\leq N+\LR(\widetilde{\cH})\leq (d+1)^d+(d+1)^{d-1}\leq (d+2)^d.
\]

For the sample complexity suppose we need $s_d(\epsilon,\delta)$ samples for a class with Littlestone dimension $d$. Then looking at the definitions of $n,k, m$ and $\alpha$ in the proof of \cref{lem:transfer}, we can pick
\[m=s_{d-1}(\eps/2,\delta/4),\quad k=O((d+1)^{4d}\log(1/\delta)),\quad 1/\alpha=O(2^d(1/\eps+m/\delta)),\]
\[n=O\left(\frac{d\log (1/\alpha)+\log(k/\delta)}{\alpha}\right)\]
and obtain the recurrence
\[s_{d}(\epsilon,\delta)\leq nk+m.\]
An elementary though tedious calculation gives the bound
\[s_{d}(\epsilon,\delta)=\frac{2^{O(d^2 \log d)}}{\eps^2\delta^{d+1}}.\]

\section{Proof of \cref{thm:main2}}\label{sec:thm_main2}

Before proceeding, we introduce several combinatorial lemmas regarding Littlestone classes, drawing inspiration from stability theory—specifically, stable regularity. 

\subsection{Goodness and truth-value hypotheses}

\begin{definition}[{\cite[Definition~3.3]{MM24},~\cite[Definition 4.4]{malliaris2025epsilon}}]
\label{def:good-distribution}
Let $0<\eta<1/2$.  A countably supported probability distribution $\mu$
on $\cH$ is \emph{$\eta$-good} if, for every $x\in\cX$,
\[
 \mu\{h:h(x)=1\}\in[0,\eta]\cup[1-\eta,1].
\]
Its \emph{truth value} $t_\mu:\cX\to\{\pm 1\}$ is the unique hypothesis satisfying
\[
 \mu\{h:h(x)=t_\mu(x)\}\ge1-\eta
 \qquad\text{for every }x\in\cX.
\]
\end{definition}

The following lemma is a learning-theoretic interpretation of a result in stable regularity.

\begin{lemma}[{\cite[Fact 2.8]{MM24}},~\cite{malliaris2014regularity}]
\label{lem:weighted-good}
Suppose $\ldim(\cH)=d$.  For every countably supported distribution
$\mu$ on $\cH$ and every $0<\eta<1/2$, there is a set
$A\subseteq\supp \mu$ such that
\[
 \mu(A)\ge\eta^d
 \qquad\text{and}\qquad
 \mu(\,\cdot\mid A)\text{ is $\eta$-good}.
\]
\end{lemma}

\begin{proof}
Start with $A_0=\supp\mu$.  If $\mu(\,\cdot\mid A_0)$ is not
$\eta$-good, some $x\in\cX$ splits $A_0$ into its two label restrictions,
each having conditional mass greater than $\eta$.  At least one restriction
has Littlestone dimension strictly less than $\ldim(A_0)$.
Replace $A_0$ by such a restriction and repeat. There can be at most $d$ drops in Littlestone dimension, so the final set $A$ has mass at least $\eta^d$.
\end{proof}

  The second ingredient we will be using in the proof of \cref{thm:main2} states the truth value of a $1/4$-good distribution can be recovered
\emph{exactly} by sampling $O(\VCdim^*(\cH))$ hypotheses.

\begin{lemma}
\label{lem:exact-majority}Let $\cH$ be a concept class with dual VC dimension $\VCdim^*(\cH)\geq 1$. There exists an integer $m=O(\VCdim^*(\cH))$
such that for every countably supported $1/4$-good distribution
$\mu$ on $\cH$ we have
\[
 \Pr_{h_1,\ldots,h_m\sim\mu}
 \bigl[\Maj(h_1,\ldots,h_m)=t_\mu\bigr]\ge\frac12.
\]
\end{lemma}

\begin{proof}
The standard VC uniform-convergence theorem, applied to the dual class
$\cH^*$, gives an $m=O(\VCdim^*{(\cH)})$ such that, with probability at least
$1/2$,
\[
 \sup_{x\in\cX}
 \left|
 \frac{1}{m}\sum_{i=1}^m\mathbbm{1}\{h_i(x)=+1\}
       -\mu\{h:h(x)=+1\}
 \right|\le\frac18.
\]
At every $x$, the frequency of $\mu$ is in the union $[0,1/4]\cup[3/4,1]$. On this event the empirical frequency therefore lies in the union $[0,3/8]\cup[5/8,1]$. Thus, every empirical
majority agrees with the truth value $t_\mu$, simultaneously for all $x\in\cX$.
\end{proof}

\subsection{A globally stable learner}
Fix $0<\eps<1$ and a deterministic proper $\operatorname{ERM}$ algorithm ERM$(\cdot)$ for $\cH$ with sample complexity $n$ such that
\begin{equation}\label{eq:dual-erm-success}
 \Pr_{S\sim D^n}
 \bigl[\loss_D(\operatorname{ERM}(S))\le\eps/2\bigr]
 \ge\frac34
\end{equation}
for every realizable $D$. Fix $m=O(\VCdim^*(\cH))$ as in \cref{lem:exact-majority}. Define

\begin{center}
 \fbox{\begin{minipage}{0.86\linewidth}
 \textbf{The majority learner $\cA_{\epsilon}$.}
 Sample $mn$ examples and collect them into $m$ independent blocks $S_1,\ldots,S_m$, each of
 size $n$.  Return
 \[
   \Maj\bigl(\operatorname{ERM}(S_1),\ldots,
             \operatorname{ERM}(S_m)\bigr).
 \]
 \end{minipage}}
\end{center}

\begin{proof}[Proof of \cref{thm:main2}.]
 We show the majority learner $\cA_{\epsilon}$ is ($\epsilon,p$)-globally stable with $p=1/2^{O(dv^*)}$. Then an application of \cref{thm:stability-list} finishes the proof.

 Fix a realizable distribution $D$, and let
$\mu$ be the output distribution of one ERM run.  Note $\supp (\mu)$ is countable. Let
\[
 G=\{h\in\supp\mu:\loss_D(h)\le\eps/2\},
 \qquad \lambda=\mu(\,\cdot\mid G).
\]
By~\eqref{eq:dual-erm-success}, $\mu(G)\ge3/4$.
Apply \cref{lem:weighted-good} with $\eta=1/4$ to $\lambda$.
It supplies $A\subseteq G$ such that
\begin{equation}\label{eq:dual-heavy-good}
 \mu(A)\ge\frac34\,4^{-d},
 \qquad \nu:=\mu(\,\cdot\mid A)\text{ is $1/4$-good}.
\end{equation}
Now the truth value hypothesis $t_\nu$ is an accurate hypothesis: whenever $t_{\nu}(x)\ne y$,
at least $3/4$ of the $\nu$-mass also predicts a label different from
$y$.  Integrating this pointwise inequality gives
\[
 \frac34\loss_D(t_{\nu})
 \le \mathbb{E}_{h\sim\nu}\loss_D(h)
 \le\frac{\eps}{2}.
\]
In particular, $\loss_D(t_{\nu})\le2\eps/3<\eps$, so we have shown accuracy for our algorithm.

The $m$ independent ERM outputs all belong to $A$ with probability
$\mu(A)^m$. Therefore,  \cref{lem:exact-majority} and
\eqref{eq:dual-heavy-good} imply
\[
 \Pr_{S\sim D^{nm}}[\cA_\eps(S)=t_{\nu}]
 \ge\frac{1}{2}\mu(A)^m
 \ge\frac12
    \left(\frac34\,4^{-d}\right)^m
 =2^{-O(dv^*)}.
\]
This completes the stability part of the proof.

Finally, it is not hard to compute the sample complexity of our globally stable algorithm as it equals the sample complexity of ERM multiplied by $m=O(v^*)$.  
\end{proof}

\bibliographystyle{alphaurl}
\bibliography{ref}

@inproceedings{CMY23,
  author    = {Zachary Chase and Shay Moran and Amir Yehudayoff},
  title     = {Stability and Replicability in Learning},
  booktitle = {2023 IEEE 64th Annual Symposium on Foundations of Computer Science (FOCS)},
  pages     = {2430--2439},
  year      = {2023},
  doi       = {10.1109/FOCS57990.2023.00148},
  url       = {https://doi.org/10.1109/FOCS57990.2023.00148}
}

@inproceedings{GGKM21,
  author    = {Badih Ghazi and Noah Golowich and Ravi Kumar and Pasin Manurangsi},
  title     = {Sample-Efficient Proper {PAC} Learning with Approximate Differential Privacy},
  booktitle = {Proceedings of the 53rd Annual ACM SIGACT Symposium on Theory of Computing (STOC)},
  pages     = {183--196},
  publisher = {Association for Computing Machinery},
  year      = {2021},
  doi       = {10.1145/3406325.3451028},
  url       = {https://doi.org/10.1145/3406325.3451028}
}

@inproceedings{Lyu25,
author = {Lyu, Xin},
title = {Private Learning of Littlestone Classes, Revisited},
year = {2026},
isbn = {9798400725364},
publisher = {Association for Computing Machinery},
address = {New York, NY, USA},
url = {https://doi.org/10.1145/3798129.3800924},
doi = {10.1145/3798129.3800924},
booktitle = {Proceedings of the 58th Annual ACM Symposium on Theory of Computing},
pages = {2218–2229},
numpages = {12},
location = {Salt Lake City, UT, USA},
series = {STOC '26}
}

@article{MM24,
  author  = {Maryanthe Malliaris and Shay Moran},
  title   = {Model Theory and Agnostic Online Learning via Excellent Sets},
  journal = {Transactions of the American Mathematical Society},
  volume  = {377},
  number  = {11},
  pages   = {7753--7776},
  year    = {2024},
  doi     = {10.1090/tran/9235},
  url     = {https://doi.org/10.1090/tran/9235}
}

@INPROCEEDINGS{BLM20,
  author={Bun, Mark and Livni, Roi and Moran, Shay},
  booktitle={IEEE 61st Annual Symposium on Foundations of Computer Science (FOCS)}, 
  title={An Equivalence Between Private Classification and Online Prediction}, 
  year={2020},
  volume={},
  number={},
  pages={389-402},
  doi = {10.1109/FOCS46700.2020.00044}
}

@inproceedings{dixon2023listandcertificate,
    author = {Dixon, Peter and Pavan, A. and Vander Woude, Jason and Vinodchandran, N. V.},
    title = {List and certificate complexities in replicable learning},
    year = {2023},
    publisher = {Curran Associates Inc.},
    address = {Red Hook, NY, USA},
    booktitle = {Proceedings of the 37th International Conference on Neural Information Processing Systems},
    articleno = {1342},
    numpages = {23},
    location = {New Orleans, LA, USA},
    series = {NeurIPS '23}
}

@article{malliaris2025epsilon,
  title={Epsilon-saturation for stable graphs and Littlestone classes},
  author={Malliaris, Maryanthe and del Campo, Olga Medrano Mart{\'\i}n and Moran, Shay},
  journal={arXiv preprint arXiv:2508.21807},
  year={2025},
  doi={10.48550/arXiv.2508.21807}
}

@article{malliaris2014regularity,
  title={Regularity lemmas for stable graphs},
  author={Malliaris, Maryanthe and Shelah, Saharon},
  journal={Transactions of the American Mathematical Society},
  volume={366},
  number={3},
  pages={1551--1585},
  year={2014},
  doi={10.1090/S0002-9947-2013-05820-5}
}

@inproceedings{chase_local_2024,
	address = {New York, NY, USA},
	series = {{STOC} 2024},
	title = {Local {Borsuk}-{Ulam}, {Stability}, and {Replicability}},
	isbn = {979-8-4007-0383-6},
	doi = {10.1145/3618260.3649632},
	urldate = {2026-06-17},
	booktitle = {Proceedings of the 56th {Annual} {ACM} {Symposium} on {Theory} of {Computing}},
	publisher = {Association for Computing Machinery},
	author = {Chase, Zachary and Chornomaz, Bogdan and Moran, Shay and Yehudayoff, Amir},
	month = jun,
	year = {2024},
	pages = {1769--1780},
}

@article{MALLIARIS2025103633,
title = {The unstable formula theorem revisited via algorithms},
journal = {Annals of Pure and Applied Logic},
volume = {176},
number = {10},
pages = {103633},
year = {2025},
issn = {0168-0072},
doi = {https://doi.org/10.1016/j.apal.2025.103633},
author = {M. Malliaris and S. Moran}
}

@inproceedings{blondal2026borsuk,
	address = {New York, NY, USA},
	series = {{STOC} '26},
	title = {Borsuk-{Ulam} and {Replicable} {Learning} of {Large}-{Margin} {Halfspaces}},
	isbn = {979-8-4007-2536-4},
	doi = {10.1145/3798129.3800771},
	urldate = {2026-06-10},
	booktitle = {Proceedings of the 58th {Annual} {ACM} {Symposium} on {Theory} of {Computing}},
	publisher = {Association for Computing Machinery},
	author = {Blondal, Ari and Hatami, Hamed and Hatami, Pooya and Lalov, Chavdar and Tretiak, Sivan},
	month = jun,
	year = {2026},
	pages = {529--540},
}

@article{alon2022private,
  title={Private and online learnability are equivalent},
  author={Alon, Noga and Bun, Mark and Livni, Roi and Malliaris, Maryanthe and Moran, Shay},
  journal={ACM Journal of the ACM (JACM)},
  volume={69},
  number={4},
  pages={1--34},
  year={2022},
  publisher={ACM, New York, NY},
   doi={10.1145/3526074}
}

@misc{hopkins2026approximate,
    title={Approximate Replicability in Learning}, 
    author={Max Hopkins and Russell Impagliazzo and Christopher Ye},
    year={2026},
    eprint={2510.20200},
    archivePrefix={arXiv},
    primaryClass={cs.LG},
    doi={10.48550/arXiv.2510.20200}
}

@inproceedings{impagliazzo2023reproducibility,
    author = {Impagliazzo, Russell and Lei, Rex and Pitassi, Toniann and Sorrell, Jessica},
    title = {Reproducibility in learning},
    year = {2022},
    isbn = {9781450392648},
    url = {https://doi.org/10.1145/3519935.3519973},
    doi = {10.1145/3519935.3519973},
    booktitle = {Proceedings of the 54th Annual ACM SIGACT Symposium on Theory of Computing},
    pages = {818–831},
    numpages = {14},
    series = {STOC 2022}
}

@inproceedings{karbasi2023reinforcement,
    author = {Karbasi, Amin and Velegkas, Grigoris and Yang, Lin and Zhou, Felix},
    booktitle = {Advances in Neural Information Processing Systems},
    doi = {10.52202/075280-3265},
    pages = {74702--74735},
    title = {Replicability in Reinforcement Learning},
    volume = {36},
    year = {2023}
}

@inproceedings{moran2023zoo,
    author = {Moran, Shay and Schefler, Hilla and Shafer, Jonathan},
    title = {The Bayesian stability zoo},
    year = {2023},
    booktitle = {Proceedings of the 37th International Conference on Neural Information Processing Systems},
    articleno = {2697},
    numpages = {22},
    location = {New Orleans, LA, USA},
    series = {NeurIPS '23},
    doi={10.52202/075280-2697}
}

@article{moran2026nearly,
  title={Nearly Linear Private Learning of Littlestone Classes},
  author={Moran, Shay and Yan, Chao},
  journal={arXiv preprint arXiv:2610.23160},
  year={2026},
  doi={10.48550/arXiv.2610.23160}
}

@inproceedings{blondal_tight_2026,
	title = {Tight list replicability bounds via a novel sphere covering theorem},
	issn = {2640-3498},
	url = {https://proceedings.mlr.press/v336/blondal26a.html},
	language = {en},
	urldate = {2026-10-02},
	booktitle = {Proceedings of {Thirty} {Ninth} {Conference} on {Learning} {Theory}},
	publisher = {PMLR},
	author = {Blondal, Ari and Hatami, Hamed and Hatami, Pooya and Lalov, Chavdar and Tretiak, Sivan},
	month = jun,
	year = {2026},
	pages = {791--807},
}

@InProceedings{hopkins25a,
  title = 	 {The Role of Randomness in Stability},
  author =       {Hopkins, Max and Moran, Shay},
  booktitle = 	 {Proceedings of the 42nd International Conference on Machine Learning},
  pages = 	 {23805--23827},
  year = 	 {2025},
  editor = 	 {Singh, Aarti and Fazel, Maryam and Hsu, Daniel and Lacoste-Julien, Simon and Berkenkamp, Felix and Maharaj, Tegan and Wagstaff, Kiri and Zhu, Jerry},
  volume = 	 {267},
  series = 	 {Proceedings of Machine Learning Research},
  month = 	 {13--19 Jul},
  publisher =    {PMLR},
  url = 	 {https://proceedings.mlr.press/v267/hopkins25a.html}
}

@InProceedings{blondal2026simplicial,
    author =	{Blondal, Ari and Hatami, Hamed and Hatami, Pooya and Lalov, Chavdar and Tretiak, Sivan},
    title =	{Simplicial Covering Dimension of Extremal Concept Classes},
    booktitle =	{17th Innovations in Theoretical Computer Science Conference (ITCS 2026)},
    pages =	{22:1--22:24},
    year =	{2026},
    volume =	{362},
    URN =		{urn:nbn:de:0030-drops-253094},
    doi =		{10.4230/LIPIcs.ITCS.2026.22},
}

@misc{blondal_sign-rank_2026,
	title = {Sign-{Rank}, {Index}, and {List} {Replicability}: {Connections} and {Separations}},
	shorttitle = {Sign-{Rank}, {Index}, and {List} {Replicability}},
	doi = {10.48550/arXiv.2606.18236},
	urldate = {2026-06-17},
	publisher = {arXiv},
	author = {Blondal, Ari and Hatami, Hamed and Hatami, Pooya and Lalov, Chavdar and Tretiak, Sivan},
	month = jun,
	year = {2026},
	note = {arXiv:2606.18236 [cs.LG]},
}

@misc{chornomaz2025spherical,
    title={Spherical dimension}, 
    author={Bogdan Chornomaz and Shay Moran and Tom Waknine},
    year={2025},
    eprint={2503.10240},
    archivePrefix={arXiv},
    primaryClass={cs.DM},
    doi={10.48550/arXiv.2503.10240} 
}

\end{document}